\documentclass{article}

     \usepackage[nonatbib,preprint]{neurips_2022}

\usepackage[utf8]{inputenc} 
\usepackage[T1]{fontenc}    
\usepackage[hidelinks]{hyperref}       
\usepackage{url}            
\usepackage{booktabs}       
\usepackage{amsfonts}       
\usepackage{nicefrac}       
\usepackage{microtype}      
\usepackage{xcolor}         

\usepackage[graphicx]{realboxes}
\usepackage{amsmath}
\usepackage{algorithm}
\usepackage[noend]{algpseudocode}
\usepackage{graphicx}
\usepackage{amsthm}
\newtheorem{theorem}{Theorem}[]

\usepackage{chngcntr}
\usepackage{apptools}

\title{Federated Bayesian Neural Regression: A Scalable Global Federated Gaussian Process}

\newcommand*\samethanks[1][\value{footnote}]{\footnotemark[#1]}

\author{%
  Haolin Yu\thanks{University of Waterloo, Waterloo, Canada; Vector Institute for AI, Toronto, Canada} \\
  \texttt{h89yu@uwaterloo.ca} \\
  \And
  Kaiyang Guo\thanks{Noah's Ark Lab, Huawei Technologies}\\
  \texttt{guokaiyang@huawei.com}\\
  \And
  Mahdi Karami\samethanks\\
  \texttt{mahdi.karami1@huawei.com}\\
  \And
  Xi Chen\samethanks\\
  \texttt{xi.chen4@huawei.com}\\
  \And
  Guojun Zhang\samethanks\\
  \texttt{guojun.zhang@huawei.com}\\
  \And
  Pascal Poupart\samethanks[1]\\
  \texttt{ppoupart@uwaterloo.ca} \\
}
\begin{document}

\maketitle

\begin{abstract}
In typical scenarios where the Federated Learning (FL) framework applies, it is common for clients to have insufficient training data to produce an accurate model. Thus, models that provide not only point estimations, but also some notion of confidence are beneficial. Gaussian Process (GP) is a powerful Bayesian model that comes with naturally well-calibrated variance estimations. However, it is challenging to learn a stand-alone global GP since merging local kernels leads to privacy leakage. To preserve privacy, previous works that consider federated GPs avoid learning a global model by focusing on the personalized setting or learning an ensemble of local models. We present Federated Bayesian Neural Regression (FedBNR), an algorithm that learns a scalable stand-alone global federated GP that respects clients' privacy. We incorporate deep kernel learning and random features for scalability by defining a unifying random kernel. We show this random kernel can recover any stationary kernel and many non-stationary kernels. We then derive a principled approach of learning a global predictive model as if all client data is centralized. We also learn global kernels with knowledge distillation methods for non-identically and independently distributed (non-i.i.d.)~clients. Experiments are conducted on real-world regression datasets and show statistically significant improvements compared to other federated GP models.

\end{abstract}

\section{Introduction}

In Federated Learning (FL) \cite{mcmahan2017communication}, we seek to train a model in a distributed way across several clients without any data leaving the clients to preserve privacy.  This is particularly attractive in application domains where each client has insufficient data to train a strong model by itself and therefore could benefit from additional information from other clients.  A trusted server is often used to aggregate the client models into a global model that improves upon the local models.  Since each client has limited data, its local model is uncertain and therefore there is value in representing this uncertainty to improve the aggregation at the server.  Intuitively, uncertain models should be given less importance in the aggregation.  Furthermore, uncertainty modeling can be used to derive confidence estimates with respect to predictions.

Gaussian Processes (GPs) with deep kernel learning~\cite{wilson2016deep, wilson2016stochastic} provide a good balance between expressiveness and complexity to represent model uncertainty.  At one end of the spectrum, most models such as traditional neural networks do not capture any uncertainty, but are simple and scalable.  At the other end of the spectrum, Bayesian neural networks express a full distribution over all weights of neural networks, but inference tends to be intractable. In between, a GP with a deep kernel consists of a neural feature extractor (also known as deep kernel) with a distribution over the last layer that facilitates exact inference.  Since the weights of the last layer are the most important for prediction, representing their uncertainty is often sufficient to capture most of the uncertainty of a model.

Several works have explored distributed GPs~\cite{deisenroth2015distributed,yin2020fedloc} and federated GPs~\cite{achituve2021personalized}.  While distributed GPs are designed to improve scalability, they pose an important privacy risk since sharing kernels either implies sharing data or sharing pairwise data similarity.  In contrast, pFedGP  \cite{achituve2021personalized} shares only the hyperparameters of deep kernels while learning local GPs that are never shared.  This personalized approach reduces privacy risks, but the local GPs do not benefit from other client information (beyond the shared kernel hyperparameters).  We propose a new federated GP technique called Federated Bayesian neural regression (FedBNR) that can learn a global GP with reduced privacy risks.  We avoid kernel sharing by working directly in the feature space and sharing scatter matrices (instead of kernels). We also propose a unifying random kernel (URK) that leverages random features and deep kernels to approximate any stationary kernel and some non-stationary kernels, including infinite kernels.  The approach is demonstrated on real world regression datasets where we achieve statistically significant improvements over prior techniques both in terms of predictions and expected calibration error.  The contributions of the paper can be summarized as follows:
\begin{itemize}
    \item New federated GP technique with deep kernel learning called federated Bayesian neural regression (FedBNR). To our knowledge, this is the first federated GP technique that learns a global GP. We describe an exact aggregation technique of the linear layer that allows inference in a way that is mathematically equivalent to inference with all the data centralized.  
    \item New unifying random kernel (URK) that provides a unifying definition for deep random kernels. URK can approximate any stationary kernel and many non-stationary kernels, including infinite kernels. This kernel has finitely many features and therefore allows us to work directly in the feature space (instead of the dual space).
\end{itemize}

\section{Related work}

\paragraph{Deep kernel learning and GP approximations.} There have been many works that committed to increase the model capacity of GPs by incorporating deep neural networks (DNNs). \cite{hinton2007using,calandra2016manifold} either pretrains a deep belief network or directly trains it with a GP to extract first-step features before sending the data into the GP with conventional kernels. \cite{wilson2016deep, wilson2016stochastic}, building on top of \cite{wilson2015kernel,wilson2015thoughts,titsias2009variational,hensman2013gaussian,nickson2015blitzkriging}, extends this idea with approximations for scalability and stochastic variational inference for classification tasks. Then \cite{tran2019calibrating} studies the variance estimations of deep kernel learning models and propose to use Monte-Carlo Dropout \cite{gal2016dropout} for better calibration, and \cite{ober2021promises} proposes to use Bayesian Neural Networks instead of deterministic DNNs to prevent over-fitting. Besides, \cite{garnelo2018conditional} designed a special architecture that makes it possible for DNNs to simulate GP behaviors. \cite{damianou2013deep} forms GPs into a deep architecture that corresponds to a deep belief network based on GP mappings. To make GPs practical, one popular method is the inducing point approximation, where the joint GP prior of training points and inducing points are approximated \cite{quinonero2005unifying}. Variants includes SoD \cite{quinonero2005unifying}, SoR \cite{smola2000sparse}, DTC \cite{seeger2003fast}, FITC \cite{FITC}, and PITC \cite{schwaighofer2002transductive}. Later, KISS-GP \cite{wilson2015kernel} gave another interpretation that inducing point approximations are equivalent to global GP interpolation, and it can exploit Kronecker structures \cite{saatcci2012scalable}. Another approximation method is random features \cite{RFF, sinha2016learning, oliva2016bayesian} that use randomized basis functions to approximate kernels. Details about random features will be covered in Section \ref{rf}, and more information about scalable GPs can be found in this survey \cite{liu2020gaussian}.

\paragraph{Distributed and Federated Gaussian Processes (GPs).} Closely related to our work is the literature on distributed and federated GPs.  Distributed GPs~\cite{deisenroth2015distributed,yin2020fedloc} were initially proposed to improve scalability by partitioning the data into several machines since GPs that operate in the dual space scale cubically with the amount of data in the worst case.  The product-of-experts framework has emerged as a popular technique to aggregate local GPs, including generalized product of experts~\cite{cao2014generalized} and robust Bayesian committee machines~\cite{deisenroth2015distributed}. Distributed optimization of hyperparameters in GPs has also been explored~\cite{xie2019distributed}.  While those techniques do not ensure data privacy, recent work about federated GPs reduce privacy risks while training in a distributed way.  This includes pFedGP~\cite{achituve2021personalized}, which optimizes the hyperparameters of a global deep kernel, while training local GPs.  In another line of work, GPs have also been used to estimate correlations between clients in FL in order to actively select independent clients for aggregation~\cite{tang2021fedgp}.  Our work differs from previous distributed and federated GPs by learning a global GP while reducing privacy risks.

\paragraph{Bayesian FL.} Beyond federated GPs, other Bayesian models have been explored to represent distributions over models and predictions in FL.  The challenge is in the aggregation of the local posteriors.  Various techniques have been proposed including posterior averaging~\cite{al2020federated}, online Laplace approximation~\cite{liu2021bayesian}, Thompson sampling~\cite{dai2020federated}, MCMC~\cite{pmlr-v151-vono22a}.













\begin{figure}
    \centering
    \includegraphics[width=0.68\textwidth]{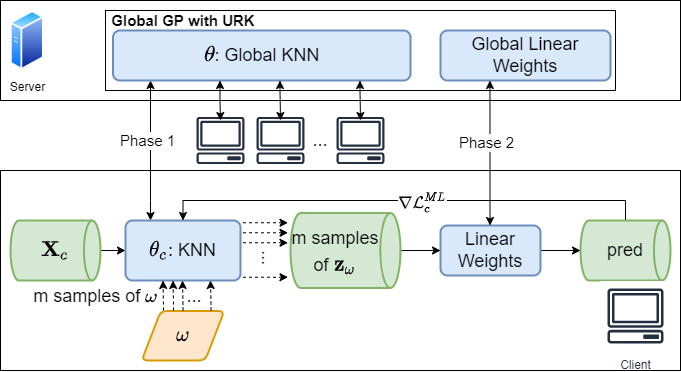}
    \caption{FedBNR learns a global federated GP in two phases: kernel learning with FL optimization and last layer updating with exact Bayesian inference. Though we work in the primal space, URK allows us to approximate an infinite kernel in the dual space with finite features in the primal space.}
    \label{fig:diag}
\end{figure}

\section{Background}


\textbf{Notation.} We will use the following notations throughout the paper. ${\bf X}, {\bf X}' \in \mathbb{R}^{p\times n}$, ${\bf x}, {\bf x}' \in \mathbb{R}^p$ are input matrices or vectors. ${\bf y} \in \mathbb{R}^n$ is the target vector. A $*$ subscript indicates the vectors have not been seen by the models. $\sigma \in \mathbb{R}$ is the noise level of Bayesian models. ${\bf I}$ is the identity matrix and $\bf 0$ is the zero vector. Their dimensions can be inferred from the context. $\phi: \mathbb{R}^{p\times \bar{a}} \rightarrow \mathbb{R}^{p' \times \bar{a}}$ denotes some basis function, and ${\bf \Phi} = \phi({\bf X})$ is the corresponding features of $\bf X$.  $k: \mathbb{R}^{p \times \bar{a}} \times \mathbb{R}^{p \times \bar{a}} \rightarrow \mathbb{R}^{\bar{a} \times \bar{a}}$ denotes a kernel function, and ${\bf K} = k({\bf X}, {\bf X})$ is the corresponding kernel matrix. $\bar{a}$ is a placeholder that indicates the function can take in matrices or vectors of any dimension. ${\bf w} \in \mathbb{R}^{p'}$ is the weights of the last linear layer in a Bayesian linear regression model. $\lambda \in \mathbb{R}$ is the prior standard deviation. 
$\mathbb{E} [\cdot], \text{Cov}(\cdot, \cdot)$, and $\text{tr}(\cdot)$ are the expectation, covariance, and trace function respectively. Any symbol with a $c$ subscript is a local version of the the original symbol, held by some client $c$.

\subsection{Gaussian process}
A GP can be informally viewed as an infinite dimensional Gaussian distribution over functions $f(\cdot)$. With a finite set of points of interest ${\bf X}$ on the support, a GP boils down to a multi-dimensional Gaussian distribution $f({\bf X})$, providing mean and variance estimates at these places. Formally, the noisy version of a GP model is established as:
\begin{equation}
    y = f({\bf x}) + \epsilon \mbox{ ,where } \epsilon \sim N(0,\sigma^2)
\end{equation}
After a prior over $f(\cdot)$ is specified, likelihood, posterior, and prediction can be computed as follows:
\begin{align}
& \mbox{Prior: } &\Pr(f(\cdot)) = N({\bf 0},k(\cdot, \cdot)) \\
& \mbox{Likelihood:} & \Pr({\bf y}|{\bf X},f(\cdot)) = N(f({\bf X}),\sigma^2 {\bf I}) \\
& \mbox{Posterior:} & \Pr(f(\cdot)|{\bf X},{\bf y}) = N(\bar{m}(\cdot), k'(\cdot, \cdot)) \\
& & \mbox{where } \bar{m}(\cdot) = k(\cdot, {\bf X})({\bf K}+\sigma^2{\bf I})^{-1}{\bf y} \mbox{,} \\
& & k'(\cdot, \cdot) = k(\cdot, \cdot)- k(\cdot, {\bf X})({\bf K}+\sigma^2{\bf I})^{-1}k({\bf X}, \cdot)\\
& \mbox{Prediction:} & \Pr(y_*|{\bf x}_*,{\bf X},{\bf y}) = N(\bar{m}({\bf x}_*), k'({\bf x}_*, {\bf x}_*))
\end{align}
The complexity of GP is cubic in the amount of data due to the inversion of ${\bf K}$. Thus in practice, full GPs are often infeasible and approximations are needed for scalability. The performance of GP models highly depends on the kernel function $k(\cdot, \cdot)$, and hyperparameters of this function can be learnt by maximizing the log marginal likelihood $\log\Pr_{\text {GP}}({\bf y}|{\bf X}) = -{\bf y}^\top({\bf K} + \sigma^2{\bf I})^{-1}{\bf y}-\log|{\bf K} + \sigma^2{\bf I}|$.

A valid kernel function is any positive definite function and can always be decomposed into the outer product of some basis functions $k({\bf x}, {\bf x}') = \phi({\bf x})^\top\phi({\bf x}')$. An important advantage of working in the dual space is that one can have a kernel corresponding to infinite features without paying a price in terms of complexity. Popular kernels such as the Gaussian kernel tend to have infinite features, providing significant model capacity. However, inference in the dual space relies on evaluating the kernel distance between data points, making it inevitably violates the privacy when data come from different sources.

If the kernel has finite features, a GP degenerates into Bayesian linear regression in the primal space:
\begin{equation}
\label{blr:begin}
    f(\cdot)={\bf w}^\top \phi(\cdot) 
\end{equation}
Given a prior, likelihood, posterior, and prediction are computed as follows:
\begin{align}
& \mbox{Prior: } &\Pr({\bf w}) = N({\bf 0},\lambda^2{\bf I}) \\
& \mbox{Likelihood:} & \Pr({\bf y}|{\bf X},{\bf w}) = N({\bf w}^\top{\bf X},\sigma^2 {\bf I}) \\
& \mbox{Posterior:} & \Pr({\bf w}|{\bf X},{\bf y}) = N(\bar{\bf w},{\bf A}^{-1}) \\
& & \mbox{where } \bar{\bf w} = \sigma^{-2}{\bf A}^{-1}{\bf \Phi}{\bf y} \mbox{ and } {\bf A} = \sigma^{-2} {\bf \Phi}{\bf \Phi}^\top + \lambda^{-2}{\bf I} \\
& \mbox{Prediction:} & \Pr(y_*|{\bf x}_*,{\bf X},{\bf y}) = N(\sigma^{-2}\phi({\bf x}_*){\bf A}^{-1}{\bf \Phi}{\bf y},\sigma^2 + \phi({\bf x}_*)^\top {\bf A}^{-1} \phi({\bf x}_*)) 
\label{blr:end}
\end{align}
The complexity of Bayesian linear regression is linear in the amount of data, but cubic in the number of features due to the inversion of ${\bf A}$. Hyperparameters of $\phi$ can be learnt similarly by maximizing the log marginal likelihood $\log\Pr_{\text {BLR}}({\bf y}|{\bf X}) = -n\log\sigma^2-\log|\lambda^2{\bf I}| - \log{|{\bf A}|} - {\bf y}^\top{\bf y}/\sigma^2+{\bf w}^\top{\bf A}{\bf w}$.

\subsection{Random features}
\label{rf}

Random features \cite{RFF} is a kind of approximation that allows working in the primal space despite the full GP having infinitely many features. The idea is to find randomized basis functions ${\bf z}$ such that:
\begin{equation}
    k({\bf x}, {\bf x}') = \phi({\bf x})^\top\phi({\bf x}') = \mathbb{E}[{\bf z}({\bf x})^\top{\bf z}({\bf x}')] \approx \left(\frac{{\bf s}^m({\bf x})}{\sqrt{m}}\right)^\top\left(\frac{{\bf s}^m({\bf x}')}{\sqrt{m}}\right)
\end{equation}
When $\frac{{\bf s}^m({\bf x})}{\sqrt{m}}$, the normalized concatenation of $m$ samples of ${\bf z}({\bf x})$, has much lower dimensionality than $\phi({\bf x})$ and the amount of data, the cubic cost in the number of features becomes negligible.

The most renowned random feature approach is random Fourier features (RFF) \cite{RFF} that can approximate any stationary kernel, based on Bochner’s theorem:

\begin{theorem}[Bochner \cite{Bochner}]
\label{boch}
A continuous kernel \(k({\bf x}, {\bf x}') = k({\bf x}-{\bf x}') = k(\delta)\) on \(\mathbb{R}^p\) is positive definite if and only if \(k(\delta)\) is the Fourier transform of a non-negative measure.
\end{theorem}

If the kernel is real-valued and properly scaled, its inverse Fourier transform $p(\omega)$ is also real-valued and is a proper probability distribution. Then a valid mapping is ${\bf z}_{\omega}({\bf x}) = [\cos(\omega^\top{\bf x}), \sin(\omega^\top{\bf x})]^\top$, since
\begin{equation}
\label{RFF_eq}
    k(\delta) = \int_{\mathbb{R}^p}p(\omega)\cos(\omega^\top\delta)d\omega = \mathbb{E}_{\omega}[\cos(\omega^\top\delta)] = \mathbb{E}_{\omega}[{\bf z}_{\omega}({\bf x})^\top{\bf z}_{\omega}({\bf x}')]
\end{equation}
The true expectation is then approximated by the empirical mean of multiple samples from ${\bf z}_{\omega}({\bf x})$, which makes it possible to recover an infinite kernel with a finite set of features, and enables working directly in the primal space.

\section{FedBNR: a scalable global federated GP}

We now describe our approach. First we extend RFF to non-stationary kernels. Then we show how we can learn a stand-alone global GP in a principled way by updating the model in two phases.

\subsection{Unifying random kernel}

Although conventional stationary kernels have been specifically popular due to their distance awareness property, deep kernel learning \cite{wilson2016deep, wilson2016stochastic} pointed out that incorporating DNNs with stationary kernels further increases the model capacity and makes it more suitable for modern machine learning tasks. However, the common architecture that a DNN is plugged in before the kernel to extract first-step features usually results in non-stationary kernels and also constrained the architecture of the combined kernel. Thus, we wish to extend RFF to non-stationary kernels and provide a unifying definition for random kernels with DNNs, with which people can design any architecture freely.

Let $\omega \in \mathbb{R}^{d'}$ be any random variable or vector, and $g: \mathbb{R}^{d'}\times\mathbb{R}^{p\times\bar{a}}\xrightarrow{} \mathbb{R}^{d\times\bar{a}}$ be any function that extracts $d$ features out of each input with some random weights $\omega$. Then we construct the random basis functions ${\bf z}$ as ${\bf z}_{\omega}({\bf x}) = g(\omega, {\bf x})$. We define the true underlying kernel and its approximation, the unifying random kernel (URK) as:
\begin{align}
\label{urk}
    k({\bf x}, {\bf x}') &= \mathbb{E}_{\omega}[{\bf z}_{\omega}({\bf x})^\top{\bf z}_{\omega}({\bf x}')] = \text{tr}(\text{Cov}_{\omega}({\bf z}_{\omega}({\bf x}), {\bf z}_{\omega}({\bf x}'))) + \mathbb{E}_{\omega}[{\bf z}_{\omega}({\bf x})]^\top\mathbb{E}_{\omega}[{\bf z}_{\omega}({\bf x}')]\\
\label{urk_2}
    &\approx \text{URK}_{\omega, g}({\bf x}, {\bf x}') = \left(\frac{{\bf s}_\omega^m({\bf x})}{\sqrt{m-1}}\right)^\top\left(\frac{{\bf s}_\omega^m({\bf x}')}{\sqrt{m-1}}\right), 
\end{align}
where ${\bf s}_\omega^m({\bf x})$ is the concatenation of $m$ samples of $\bf{z}_\omega({\bf x})$. URK can also recover any stationary kernel since RFF is a special case of it:

\begin{theorem}
Given any properly scaled stationary kernel \(k({\bf x}, {\bf x}')\) on \(\mathbb{R}^p\) and its inverse Fourier transform \(p(\omega)\), \(\exists \omega \sim p(\omega), g(\omega, {\bf x}) = [\cos(\omega^\top{\bf x}), \sin(\omega^\top{\bf x})]^\top\) s.t. \(\lim_{m\xrightarrow{}\infty} \text{URK}_{\omega, g}({\bf x}, {\bf x}') = k({\bf x}, {\bf x}')\). 
\end{theorem}
\begin{proof}
Following Theorem \ref{boch} and Equation \ref{RFF_eq}, the construction of \(\omega\) and \({\bf z}_\omega\) is equivalent to RFF, so \(k({\bf x}, {\bf x}') = \mathbb{E}_{\omega}[{\bf z}_{\omega}({\bf x})^\top{\bf z}_{\omega}({\bf x}')]\). By Equation \ref{urk_2}, \(\text{URK}_{\omega, g}({\bf x}, {\bf x}')\) is an unbiased estimator of \(\mathbb{E}_{\omega}[{\bf z}_{\omega}({\bf x})^\top{\bf z}_{\omega}({\bf x}')]\), so \(\lim_{m\xrightarrow{}\infty} \text{URK}_{\omega, g}({\bf x}, {\bf x}') = k({\bf x}, {\bf x}')\).
\end{proof}

However, note that in the definition of URK we do not rely on the inverse Fourier transformation or the Bochner's theorem to find a valid distribution for $\omega$. Instead, any $\omega$ and $g$ can give us a valid kernel:

\begin{theorem}
\label{URF:all}
Given any proper probability distribution \(p(\omega)\) on \(\mathbb{R}^{d'}\) and function \(g\) on \(\mathbb{R}^{d'}\times\mathbb{R}^{p\times\bar{a}}\xrightarrow{} \mathbb{R}^{d\times\bar{a}}\), the corresponding kernel matrix \(k({\bf X}, {\bf X}') = \lim_{m\xrightarrow{}\infty} \text{URK}_{\omega, g}({\bf X}, {\bf X}')\) is positive definite.
\end{theorem}
\begin{proof}
Following Equation \ref{urk}, we have $$\textstyle k({\bf X}, {\bf X}') = \sum_{i=1}^{d}(\text{Cov}_{\omega}({\bf z}_{\omega}({\bf X})_i, {\bf z}_{\omega}({\bf X}')_i)) + \mathbb{E}_{\omega}[{\bf z}_{\omega}({\bf X})]^\top\mathbb{E}_{\omega}[{\bf z}_{\omega}({\bf X}')].$$ ${\bf z}_{\omega}({\bf X})_i$ denotes the $i^{th}$ row of the $d\times n$ matrix ${\bf z}_{\omega}({\bf X})$. The first term is the addition of $d$ covariance matrices and is always positive definite; the second term is symmetrical and essentially adds one feature to the first term. Thus $k({\bf X}, {\bf X}')$ is always positive definite.
\end{proof}

Since $g$ is an arbitrary function, we can assign it any DNN with any architecture. We call it a Kernel Neural Network (KNN) since its weights, $\theta$, are essentially the kernel hyperparameters. Theorem \ref{URF:all} allows us to train kernels with optimization methods similarly to training DNNs from a much richer hypothesis set than conventional kernels. To provide some insights into possible non-stationary kernels expressed by URK, we give an example construction of a kernel with infinite features below.

Let $\omega \sim N({\bf 0}, {\bf I})$, $g(\omega, {\bf x}) = \exp({\omega^\top{\bf x}})$. By definition, $k({\bf x},{\bf x}') = \mathbb{E}_\omega[\exp(\omega^\top({\bf x}+{\bf x}'))] = M_\omega({\bf x}+{\bf x}') = \exp(({\bf x}+{\bf x}')^\top({\bf x}+{\bf x}')/2)$, where $M_\omega$ is the moment generating function of Gaussian distributions. This kernel contains features of all polynomial kernels with $c=0$, times $\exp({\bf x}^\top{\bf x}/2)$ and a constant. The proof can be found in Appendix A.

\subsection{Two-phase update}

The training procedure of FedBNR can be divided into 2 phases, as illustrated in Fig.~\ref{fig:diag}. In the first phase, we train the KNN by optimization methods. In the second phase, we calculate the weights of the last linear layer that maps the random features to the output space, by a closed-form formula inferred from Equation \ref{blr:begin} - \ref{blr:end}. The pseudocode of FedBNR is summarized in Algorithm~\ref{alg:alg}.

\begin{algorithm}
\caption{Federated Bayesian Neural Regression \\($\theta$: the global KNN, $\sigma$: noise level, $\lambda$: prior covariance of the linear layer, $\omega$: a set of random numbers, $\zeta$: local learning step, $\zeta'$: knowledge distillation step)}\label{alg:cap}
\label{alg:alg}
\begin{algorithmic}
\State \textbf{Phase 1: Kernel Learning}
\State \text{Initialize shared kernel} $\theta \gets \theta^0$\text{, hyperparameters} $\sigma \gets \sigma^0, \lambda \gets \lambda^0$
\For{each aggregation round $t \gets 0, 1, 2, \cdots$}
    \For{each client $c \in S$}
    \State $\theta_c^t, \sigma_c^t, \lambda_c^t \gets \theta^t, \sigma^t, \lambda^t$
        \For {each local update round $k \gets 0, 1, 2, \cdots$}
        \State $\theta_c^t, \sigma_c^t, \lambda_c^t \gets -\zeta\nabla \mathcal{L}^{ML}_c$ according to Equation \ref{LML}
        \EndFor
    \EndFor
    \If {FedAvg}
    \State $\theta^{t+1}, \sigma^{t+1}, \lambda^{t+1} \gets mean(\theta_S^{t}), mean(\lambda_S^{t}), mean(\lambda_S^{t})$
    \ElsIf {Knowledge Distillation}
        \For{each knowledge distillation round $k' \gets 0, 1, 2, \cdots$}
        \State $\theta^{t+1}, \sigma^{t+1}, \lambda^{t+1} \gets -\zeta'\nabla\mathcal{L}^{KD}$ according to Equation \ref{LKD}
        \EndFor
    \EndIf
\EndFor
\State \textbf{Phase 2: Update the Global Linear Layer}
\State \textbf{Server:} send $\theta, \sigma, \lambda, \omega$ to all clients
\For{each {\bf Client} $c \in S$}
\State compute the random features ${\bf \Phi}_c \gets \theta(\omega, {\bf X}_c)$ and send ${\bf \Phi}_c {\bf \Phi}_c^\top$ to the server
\EndFor
\State \textbf{Server:} send ${\bf A}^{-1} \gets (\sigma^{-2} \sum^c{({\bf \Phi}_c {\bf \Phi}_c^\top)} + \lambda^{-2}{\bf I})^{-1}$ to all clients
\For{each {\bf Client} $c \in S$}
\State $\bar{\bf w}_c \gets \sigma^{-2}{\bf A}^{-1}{\bf \Phi}_c{\bf y}_c$ and send $\bar{\bf w}_c$ to the server
\EndFor
\State \textbf{Server:} $\bar{\bf w} \gets \sum_{c \in S}{\bar{\bf w}_c}$
\end{algorithmic}
\end{algorithm}

In phase 1, we follow a standard training procedure under the FL framework. We assume there is a central server holding the shared global KNN weights $\theta$ and global hyperparameters $\sigma, \lambda$ that denote the noise level and the prior variance respectively. We assume there is a set of clients $c \in S$ holding local KNNs weights $\theta_c$, local hyperparameters $\sigma_c, \lambda_c$, local inputs ${\bf X}_c$, and local targets ${\bf y}_c$. We use $\theta({\bf x})$ to denote the result of sending $\bf x$ through the KNN. In the beginning of each aggregation round, the server first sends a copy of aggregated or initialized $\theta, \sigma, \lambda$ to all the clients. Then all the clients $c \in S$ first update their local model for a fixed number of iterations. Then they send $\theta_c$, $\sigma_c$, and $\lambda_c$ back to the server for aggregation and starts another aggregation round. The local loss function is the local log marginal likelihood:
\begin{equation}
\label{LML}
    \mathcal{L}^{ML}_c = \log \text{Pr}_{\text{BLR}}({\bf y}_c|\theta_c({\bf X}_c); \sigma_c, \lambda_c) = -n_c\log\sigma_c^2-\log(|\lambda_c^2{\bf I}||{\bf A}_c|) - {\bf y}_c^\top{\bf y}_c/\sigma_c^2+{\bf w}_c^\top{\bf A}_c{\bf w}_c
\end{equation}
One commonly used method for aggregation is the FedAvg \cite{FedAvg} heuristic, where the new global model parameters are assigned the average of all client model parameters $\theta = \sum_{c \in S} \theta_c / |S|$. However, multiple works \cite{karimireddy2020scaffold, shoham2019overcoming,mohri2019agnostic} have pointed out the quality and the convergence rate of this heuristic can suffer from non-i.i.d.~clients. To account for this, we propose to adapt kernel knowledge distillation \cite{he2021feature} to aggregate the KNNs. We assume the server holds a relatively small dataset ${\bf X}_{kd}, {\bf y}_{kd}$ and tries to minimize the following knowledge distillation loss with respect to this dataset:
\begin{equation}
\label{LKD}
    \mathcal{L}^{KD} = \mathcal{L}^{ML}_{kd} + \alpha*\text{MSE}\left(\theta({\bf X}_{kd})^\top\theta({\bf X}_{kd}) - \sum_{c \in S}\theta_c({\bf X}_{kd})^\top\theta_c({\bf X}_{kd})/|S|\right)
\end{equation}
Here, $\mathcal{L}^{ML}_{kd}$ is the global log marginal likelihood loss $\mathcal{L}^{ML}$ with respect to the knowledge distillation dataset ${\bf X}_{kd}, {\bf y}_{kd}$. $\alpha$ is a common hyperparameter in knowledge distillation methods to adjust the ratio between the log marginal likelihood loss and the mean squared error (MSE) loss. The MSE loss factor forces the global kernel $\theta({\bf X}_{kd})^\top\theta({\bf X}_{kd})$ to simulate the mean of all client kernels $\sum_{c \in S}\theta_c({\bf X}_{kd})^\top\theta_c({\bf X}_{kd})/|S|$, which is akin to concatenating all the features of the client kernels. Ideally, if the global kernel successfully learns to do so, it should not perform worse on any of the clients, while the FedAvg heuristic has no similar guarantees.

In phase 2, we fix the kernel hyperparameters and learn ${\bf A}^{-1}$, the matrix for covariance prediction, and $\bar{\bf w}$, the weights of the last linear layer, in an exact way as if all client data are centralized. To understand the procedure, first notice that we can decompose ${\bf A}$ and $\bar{\bf w}$ as follows:
\begin{equation}
\label{A}
    {\bf A} = \sigma^{-2}{\bf \Phi}{\bf \Phi}^\top + \lambda^{-2}{\bf I} = \sigma^{-2}\sum_{c \in S}{\bf \Phi}_c{\bf \Phi}_c^\top + \lambda^{-2}{\bf I}
\end{equation}
\begin{equation}
\label{w}
    \bar{\bf w} = \sigma^{-2}{\bf A}^{-1}{\bf \Phi}{\bf y} = \sum_{c \in S}\sigma^{-2}{\bf A}^{-1}{\bf \Phi}_c{\bf y}_c
\end{equation}
Here ${\bf \Phi}_c = \theta({\bf X}_c)$ denotes the random features of local inputs extracted by the global KNN. The server first broadcasts the global model to all the clients, and asks them to return the scatter matrices ${\bf \Phi}_c{\bf \Phi}_c^\top$, and then the server can calculate $\bf{A}$ following Equation \ref{A}. Next, the server broadcasts ${\bf A}^{-1}$ and asks clients for the intermediate weights $\sigma^{-2}{\bf A}^{-1}{\bf \Phi}_c{\bf y}_c$. Finally, the server can calculate $\bar{\bf w}$ according to Equation \ref{w} and broadcasts the whole model again to all the clients.

We claim that FedBNR protects privacy of clients at least as well as FedAvg and other federated learning algorithms that send client models to the server. In phase 1, the aggregation only requires client model parameters. In phase 2, we send information twice outside each client: the scatter matrices and the intermediate weights. Sending these matrices and vectors are safer than sending the features ${\bf \Phi}_c$ directly since they have limited sizes that are completely independent of the training data size $n_c$, meaning they must only contain limited information about the raw data. Specifically, the scatter matrices are of size $md \times md$, and the intermediate weights are of size $md$, where $m$ is the number of samples from ${\bf z}$ and $d$ is the output dimension of the KNN.

\section{Experiments}

\subsection{Synthetic experiment}

Although personalized federated learning (PFL) is usually viewed as an advanced version of the plain FL framework since it learns a fine-tuned local model for each client and can automatically handle preference distribution skew \cite{zhu2021federated} (i.e. $\Pr({\bf y}|{\bf X})$ varies for clients), it is noteworthy that if all the clients can agree on a single global model in the hypothesis set, PFL with no global model may not be the best choice due to the trade-off between generalization and personalization. For example, a hospital that only collected data for cancer may also want their model to help diagnosing COVID-19, but personalization would prevent the model from generalizing to other ranges. Moreover, when a new client comes in with few data points, the quality of prediction will suffer compared to other clients, even only querying its own range.

\begin{figure}
  \centering
  \includegraphics[width=0.45\textwidth]{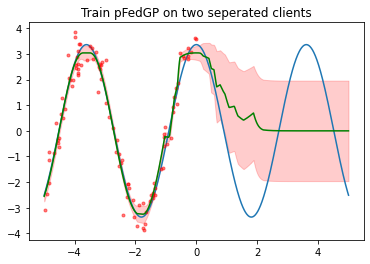}\includegraphics[width=0.45\textwidth]{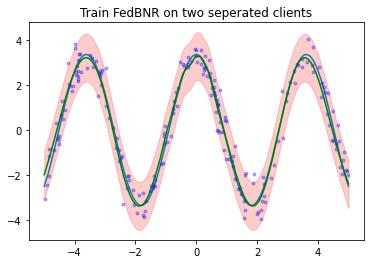}
  \includegraphics[width=0.45\textwidth]{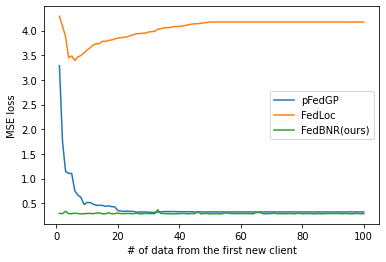}\includegraphics[width=0.45\textwidth]{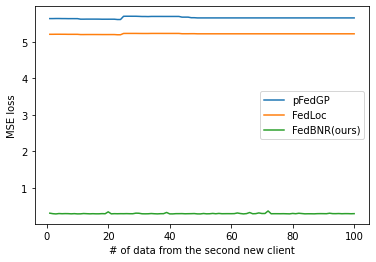}
  \caption{Top graphs: (i) left: prediction of a pFedGP model trained on two non-overlapping clients; (ii) right: global prediction of FedBNR. Blue curve: underlying truth; green curve: point estimation; red range: 95\% confidence interval; dots: training data leveraged by the model. Bottom graphs: MSE of predictions tested on the same range v.s. number of training data from a new client; (i) left: client with training range [-5, 5]; (ii) right: client with training range [5, 15]}
  \label{fig:syn}
\end{figure}

We designed the following synthetic experiments to support these arguments. We first decide a true underlying function, sample 200 points uniformly from the range $[-5, 5]$, and add Gaussian noise with $\sigma = 0.5$. We then learn the kernel hyperparameters in a centralized fashion to eliminate any impact from imperfect kernels later on. Details about the kernel sizes and architectures can be found in Appendix B. We then assign the first 100 points in range $[-5, 0]$ to client 1, and the rest to client 2. We train pFedGP and FedBNR with the learnt kernel hyperparameters fixed and query for prediction in range $[-5, 5]$. The top left graph of Figure \ref{fig:syn} shows the result of pFedGP, and the top right one shows the result of FedBNR. The blue curve shows the true underlying function, and the green curve shows the predictions. The light red area is a 95\% confidence interval based on variance. 

Since both algorithms share the same kernel hyperparameters, the only difference lies in whether they leverage data of both clients for the predictive model. As shown in the pFedGP graph, the personalized model of client 1 only sees its own data (red dots), so it does not generalize well to the range of client 2, while the global model learnt by FedBNR can leverage all the data (blue dots) and generalizes to the full range.

Next, we introduce two new clients into the system. The first client holds data uniformly sampled from range $[-5, 5]$, and the second from range $[5, 15]$. We again fix the kernel hyperparameters learnt centrally and train pFedGP, FedLoc, and FedBNR on these new clients seperated. For testing, we still query the range $[-5, 5]$. The bottom graphs of Figure \ref{fig:syn} shows the MSE loss as the size of data of both new clients grows. As expected, for pFedGP, the quality of prediction is massively impacted when there are few points in the first case, and for FedBNR the loss remains approximately a straight line. For FedLoc, the loss is also impacted due to zero-out effects of non-overlapping client models. Even worse, when training data and testing data are not in the same range for the second new client, the MSE loss of both pFedGP and FedLoc never drops back to the level of previous clients.

\subsection{UCI regression datasets}

We conducted comprehensive experiments on five UCI regression datasets under ten cases. Results of two variants are reported: i) \textit{FedBNR} that performs the FedAvg heuristic at aggregation; and ii) \textit{FedBNR-KD} that performs the knowledge distillation method at aggregation. 

Two groups of baselines are compared to our method for RMSE error: i) ablation study that contains local+local, local+global, avg+local, and kd+global, in the format of kernel learning method + last linear weight learning method, where we remove the global aggregation of either phase 1 or phase 2 from our methods; ii) previous works that contains (1) FedAvg \cite{FedAvg}, a standard non-Bayesian FL algorithm that has a global model; (2) FedProx \cite{FedProx}, a non-Bayesian FL algorithm that adds a proximal loss to FedAvg to prevent client models from getting too far from the global model; (3) pFedGP \cite{achituve2021personalized}, a Bayesian PFL method that learns a local GP with a shared deep kernel for each client; and (4) pFedGP \cite{yin2020fedloc}, a Bayesian FL method that directly applies distributed GP methods \cite{xie2019distributed} without deep kernel learning. We also compare to pFedGP and FedLoc for calibration errors since they are Bayesian models that have a notion of confidence. Besides, we report results of a centralized GP equipped with URK as a casual reference of the testing error lower-bound.

We used fully connected neural networks to extract first-step features for FedAvg, FedProx, and pFedGP. We used Gaussian kernels for the GPs in pFedGP and FedLoc. For fairness, KNNs used by our methods have similar architectures to the combined kernel of pFedGP. For scalability, FITC \cite{FITC} approximations are implemented for the other GPs as described in pFedGP. We used 50 random samples for KNN and 50 inducing points for pFedGP and FedLoc. Further details are in Appendix B. 

Each dataset is uniformly divided into 8:1:1 training:testing:validation sets globally. The training data is sorted by the feature that has the largest absolute correlation coefficient with the output, and divided into multiple chunks. Each client randomly takes two chunks so that their data distributions are heterogeneous. The larger the absolute correlation coefficient is, the more significant the distribution skew is. Before training, we tune hyperparameters that cannot be learnt by gradient descent with grid searching on the validation set. Specially, FedBNR-KD uses 80\% of the validation set for knowledge distillation, and the rest 20\% for validation. All the datasets are then ran for at most 50 local epochs times 100 aggregation rounds in full batches. Validation and testing error are recorded for each aggregation round. For methods that contain only local models, these errors are defined as the mean error of all local models with respect to the testing/validation set. If the validation error has not improved for 5 rounds, the training process is terminated. We report the average minimum testing RMSE for 10 random seeds for each case in Table \ref{Tab:UCI1}. We also measured the statistical significance of the results compared to FedBNR with one-tailed Wilcoxon signed-rank tests \cite{wilcoxon}. We then report expected calibration errors in Table \ref{Tab:UCI2} and perform the Wilcoxon test compared to FedBNR-KD. The maximum calibration error, the Brier score, and further details of Table \ref{Tab:UCI1} are included in Appendix C.

\begin{table}
{\small
  \caption{UCI regression datasets, RMSE reported. $\Uparrow$ and $\uparrow$ denote significantly worse results with $p < 0.01$ and $p < 0.05$ respectively; $\Downarrow$ and $\downarrow$ denote significantly better results similarly.}
  \centering
  \begin{tabular}{lllllllllll}
    \toprule
    & \multicolumn{2}{c}{Skillcraft \cite{thompson2013video}} & \multicolumn{2}{c}{SML \cite{zamora2014line}} & \multicolumn{2}{c}{Parkinsons \cite{tsanas2009accurate}} & \multicolumn{2}{c}{Bike \cite{ve2020rule}} & \multicolumn{2}{c}{CCPP \cite{tufekci2014prediction}}\\
    \cmidrule(r){2-3}
    \cmidrule(r){4-5}
    \cmidrule(r){6-7}
    \cmidrule(r){8-9}
    \cmidrule(r){10-11}
    train/test size & \multicolumn{2}{c}{2670 334} & \multicolumn{2}{c}{3309 414} & \multicolumn{2}{c}{4700 587} & \multicolumn{2}{c}{7008 876} & \multicolumn{2}{c}{7654 957}\\
    corr-coef & \multicolumn{2}{c}{-0.660} & \multicolumn{2}{c}{0.783} & \multicolumn{2}{c}{0.410} & \multicolumn{2}{c}{0.539} & \multicolumn{2}{c}{-0.948}\\
    \#clients     & 10 & 100 & 10 & 100 & 10 & 100 & 10 & 100 & 10 & 100 \\
    \midrule
    Central GP & \multicolumn{2}{c}{0.95} & \multicolumn{2}{c}{0.21} & \multicolumn{2}{c}{3.58} & \multicolumn{2}{c}{0.37} & \multicolumn{2}{c}{4.02}  \\
    \midrule
    local+local & 1.26$^\Uparrow$ & 1.48$^\Uparrow$ & 1.49$^\Uparrow$ & 2.32$^\Uparrow$ & 10.9$^\Uparrow$ & 10.6$^\Uparrow$ & 0.79$^\Uparrow$ & 0.92$^\Uparrow$ & 14.6$^\Uparrow$ & 19.3$^\Uparrow$\\
    local+global & 1.08$^\Uparrow$ & 1.22$^\Uparrow$ & 1.00$^\Uparrow$ & 1.65$^\Uparrow$ & 6.42$^\Uparrow$ & 7.42$^\Uparrow$ & 0.59$^\Uparrow$ & 0.73$^\Uparrow$ & 5.62$^\Uparrow$ & 7.03$^\Uparrow$\\
    avg+local & 1.05$^\Uparrow$ & 1.06$^\Uparrow$ & 0.61$^\Uparrow$ & 0.81$^\Uparrow$ & 9.84$^\Uparrow$ & 8.80$^\Uparrow$ & 0.45$^\Uparrow$ & 0.54$^\Uparrow$ & 8.32$^\Uparrow$ & 13.4$^\Uparrow$\\
    kd+local & 1.04$^\Uparrow$ & 1.28$^\Uparrow$ & 0.86$^\Uparrow$ & 1.34$^\Uparrow$ & 6.15$^\Uparrow$ & 6.97$^\Uparrow$ & 0.51$^\Uparrow$ & 0.61$^\Uparrow$ & 10.0$^\Uparrow$ & 17.1$^\Uparrow$\\
    \midrule
    FedAvg \cite{FedAvg} & 1.00$^\uparrow$ & 1.03$^\Uparrow$ & 0.36$^\Uparrow$ & 0.71$^\Uparrow$ & 6.83$^\Uparrow$& 7.38$^\Uparrow$ & \bf{0.38}$^\downarrow$ & 0.42 & 4.45 & 4.44\\
    FedProx \cite{FedProx} & 0.98 & 1.05$^\Uparrow$ & 0.34$^\uparrow$ & 0.62$^\Uparrow$ & 6.32$^\Uparrow$& 7.38$^\Uparrow$ & 0.39 & 0.42 & 4.43 & 4.47\\
    pFedGP \cite{achituve2021personalized} & 0.99$^\uparrow$ & 1.15$^\Uparrow$ & 0.75$^\Uparrow$ & 1.34$^\Uparrow$ & 9.36$^\Uparrow$ & 8.91$^\Uparrow$ & 0.45$^\Uparrow$ & 0.46$^\Uparrow$ & 14.2$^\Uparrow$ & 18.2$^\Uparrow$\\
    FedLoc \cite{yin2020fedloc} & 1.08$^\Uparrow$ & 4.15$^\Uparrow$ & 2.83$^\Uparrow$ & 5.43$^\Uparrow$ & 8.40$^\Uparrow$ & 11.5$^\Uparrow$ & 0.64$^\Uparrow$ & 0.75$^\Uparrow$ & 20.6$^\Uparrow$ & 44.1$^\Uparrow$\\
    \midrule
    \textbf{ours}\\
    FedBNR & 0.98 & \bf{0.97} & \bf{0.25} & \bf{0.44} & \bf{3.10} & 5.42 & 0.39 & \bf{0.42} & 4.40 & 4.51\\
    FedBNR-KD & \bf{0.96}$^\Downarrow$ & 0.98$^\uparrow$ & 0.55$^\Uparrow$ & 0.55$^\Uparrow$ & 4.58$^\Uparrow$ & \bf{4.67}$^\downarrow$ & 0.43$^\Uparrow$ & 0.48$^\Uparrow$ & \bf{4.38} & \bf{4.38}$^\downarrow$\\
    
    \bottomrule
  \end{tabular}
\label{Tab:UCI1}
}
\end{table}

The results show: i) in terms of RMSE, our methods are statistically better than the Bayesian models pFedGP and FedLoc in all the cases and better than the non-Bayesian models FedAvg and FedProx in most of the cases, especially when the training set at each client is significant smaller than the whole set; ii) compared with the ablation study methods, our methods always perform better, so the global aggregation at both phases are essential; iii) FedLoc without deep kernel learning has a especially smaller model capacity; iv) FedBNR-KD only outperforms FedBNR in 40\% of cases in terms of RMSE, which is probably due to the small size of data (8\% of all) used for knowledge distillation. However, FedBNR-KD is clearly more stable when client heterogeneity gets worse as the number of clients increases. It is also better calibrated than FedBNR and other Bayesian models in most cases.

\begin{table}
{\small
  \caption{Expected calibration error (ECE) reported. $\Uparrow$ and $\uparrow$ denote significantly worse results with $p < 0.01$ and $p < 0.05$ respectively; $\Downarrow$ and $\downarrow$ denote significantly better results similarly.}
  \centering
  \begin{tabular}{lllllllllll}
    \toprule
    & \multicolumn{2}{c}{Skillcraft} & \multicolumn{2}{c}{SML} & \multicolumn{2}{c}{Parkinsons} & \multicolumn{2}{c}{Bike} & \multicolumn{2}{c}{CCPP}\\
    \cmidrule(r){2-3}
    \cmidrule(r){4-5}
    \cmidrule(r){6-7}
    \cmidrule(r){8-9}
    \cmidrule(r){10-11}
    \#clients     & 10 & 100 & 10 & 100 & 10 & 100 & 10 & 100 & 10 & 100 \\
    \midrule
    Central GP & \multicolumn{2}{c}{0.02} & \multicolumn{2}{c}{0.09} & \multicolumn{2}{c}{0.27} & \multicolumn{2}{c}{0.11} & \multicolumn{2}{c}{0.24}  \\
    \midrule
    
    pFedGP \cite{achituve2021personalized} & 0.43$^\Uparrow$ & 0.38$^\Uparrow$ & 0.45$^\Uparrow$ & 0.45$^\Uparrow$ & 0.49$^\Uparrow$ & 0.48$^\Uparrow$ & 0.42$^\Uparrow$ & 0.41$^\Uparrow$ & 0.49$^\Uparrow$ & 0.49$^\Uparrow$\\
    FedLoc \cite{yin2020fedloc} & 0.32$^\Uparrow$ & 0.44$^\Uparrow$ & \bf{0.12}$^\Downarrow$ & 0.27$^\Uparrow$ & 0.32$^\Uparrow$ & 0.43$^\Uparrow$ & 0.26$^\Uparrow$ & 0.16$^\Uparrow$ & \bf{0.23}$^\Downarrow$ & 0.50$^\Uparrow$\\
    \midrule
    \textbf{ours}\\
    FedBNR & \bf{0.05} & 0.20$^\Uparrow$ & 0.39$^\Uparrow$ & 0.37$^\Uparrow$ & 0.36$^\uparrow$ & 0.40$^\Uparrow$ & \bf{0.07} & 0.04$^\Uparrow$ & 0.24$^\Downarrow$ & \bf{0.20}$^\Downarrow$\\
    FedBNR-KD & {\bf 0.05} & {\bf 0.06} & 0.20 & \bf{0.21} & \bf{0.29} & \bf{0.30} & 0.08 & \bf{0.09} & 0.30 & 0.31\\
    
    \bottomrule
  \end{tabular}
\label{Tab:UCI2}
}
\end{table}

\section{Conclusion}

In this work, we proposed FedBNR, a novel Bayesian federated learning algorithm that learns a global federated GP without privacy leakage and introduced URK, a unifying definition for deep random features, to approximate kernels with randomized basis functions in the primal space. FedBNR learns a kernel represented by a DNN under the URK definition, and share scatter matrices instead of direct features to achieve the exact global optimum of the last layer. We derived two variants based on the FedAvg heuristic and the knowledge distillation. Both variants shows empirically statistically significant improvements in terms of point estimation and calibration than other federated GP models. 


\begin{ack}


Resources used in preparing this research at the University of Waterloo were provided by Huawei Canada, the province of Ontario and the government of Canada through CIFAR and companies sponsoring the Vector Institute.

\end{ack}



{
\small

\bibliography{main.bib}
\bibliographystyle{unsrt}
}

\newpage

\AtAppendix{\counterwithin{theorem}{section}}
\appendix

\renewcommand{\thefigure}{A.\arabic{figure}}
\setcounter{figure}{0}

\begin{figure}
  \centering
  \includegraphics[width=0.93\textwidth]{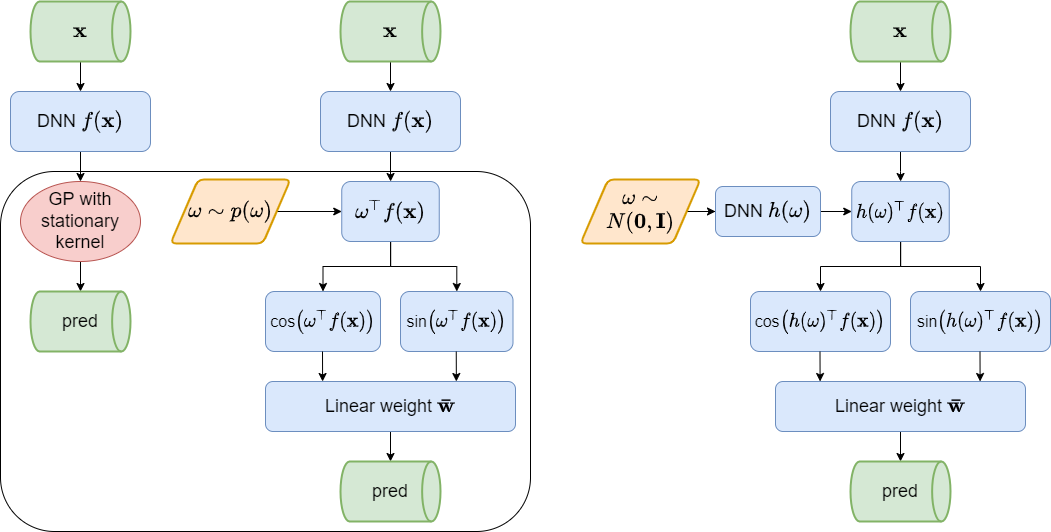}
  \caption{From left to right: a standard deep kernel learning algorithm with conventional stationary kernel GPs; the corresponding URK architecture; a more general architecture enabled by URK for convenient latent stationary kernel learning.}
  \label{fig:app}
\end{figure}

\section{Unifying random kernel}

\subsection{Example construction of a non-stationary kernel}

\begin{theorem}
Let \(\omega \sim N({\bf 0}, {\bf I}), g(\omega, {\bf x}) = \exp({\omega^\top{\bf x}})\), then \(k({\bf x},{\bf x}') = \lim_{m\xrightarrow{}\infty} \text{URK}_{\omega, g}({\bf x}, {\bf x}')\) has infinite features.
\end{theorem}

\begin{proof}
    \begin{align*}
        k({\bf x},{\bf x}') &= \exp(({\bf x}+{\bf x}')^\top({\bf x}+{\bf x}')/2)\\
        &= \exp(({\bf x}^\top{\bf x} + {\bf x}'^\top{\bf x}')/2)\exp({\bf x}^\top{\bf x}')\\
        &= \exp(({\bf x}^\top{\bf x} + {\bf x}'^\top{\bf x}')/2)\sum_{l=0}^{\infty}\frac{({\bf x}^\top{\bf x}')^l}{l!}\text{, by the Maclaurin series of }\exp(x)\\
        &= \sum_{l=0}^{\infty}\left(\frac{\exp({\bf x}^\top{\bf x}/2)}{\sqrt{l!}}({\bf x}^\top{\bf x}')^l\frac{\exp({\bf x}'^\top{\bf x}'/2)}{\sqrt{l!}}\right)\\
        &= \phi({\bf x})^\top\phi({\bf x}')
    \end{align*}
\end{proof}
Additionally, we show URK can recover the popular polynomial kernel, a non-stationary kernel beyond RFF's capability.
\begin{theorem}
Let \(c_{poly}, n_{poly} \in \mathbb{R}\). Define \({\bf p}_{poly} = [\frac{1}{2},\ \frac{1}{2p},\ \frac{1}{2p},\ \cdots,\ \frac{1}{2p}]^\top \in \mathbb{R}^{p+1},\ \omega \sim Multi(n_{poly}, {\bf p}_{poly})\), the multinomial distribution, \(\bar{\bf x} = [\sqrt{2c_{poly}},\ \sqrt{2p}{\bf x}^\top]^\top \in \mathbb{R}^{p+1},\ g(\omega, {\bf x}) = \exp({\omega^\top\log\bar{\bf x}})\), then \(k({\bf x},{\bf x}') = \lim_{m\xrightarrow{}\infty} \text{URK}_{\omega, g}({\bf x}, {\bf x}') = ({\bf x}^\top{\bf x}'+c_{poly})^{n_{poly}}\)
\end{theorem}
\begin{proof}
In the following proof, any subscript \(i\) means the \(i_{th}\) entry of the vector. By definition,
    \begin{align*}
        k({\bf x},{\bf x}') &= \mathbb{E}_{\omega}[{\bf z}_{\omega}({\bf x})^\top{\bf z}_{\omega}({\bf x}')]\\
        &= \mathbb{E}_{\omega}[\exp\left(\omega^\top(\log\bar{\bf x} + \log\bar{\bf x}')\right)]\\
        &= M_{\omega}(\log\bar{\bf x} + \log\bar{\bf x}')\text{, the moment generating function of }\omega\\
        &= \left(\sum_{i=1}^{p+1}{\bf p}_{poly, i}\exp(\log\bar{\bf x}_i + \log\bar{\bf x}'_i)\right)^{n_{poly}}\\
        &= \left(\frac{1}{2}\exp(2\log\sqrt{2c_{poly}}) + \sum_{i=1}^{p}\frac{1}{2p}\exp(\log\sqrt{2p}{\bf x}_i + \sqrt{2p}\log{\bf x}'_i)\right)^{n_{poly}}\\
        &= \left(c_{poly} + \sum_{i=1}^{p}{\bf x}_i{\bf x}'_i\right)^{n_{poly}}\\
        &= ({\bf x}^\top{\bf x}'+c_{poly})^{n_{poly}}
    \end{align*}
\end{proof}

\subsection{Greater expressiveness with URK}


We expand on new architectures of deep random kernels enabled by the definition of URK in this section to show that URK is more flexible than common heuristics in deep kernel learning. Minimal arguments and evidence are provided below since the ultimate goal of this paper is still to propose a Bayesian FL algorithm that can learn a global GP, not a random feature algorithm that provides better GP approximation.

As illustrated in Figure \ref{fig:app}, URK can recover any standard deep kernel combined with conventional stationary GPs easily. Further more, we can exploit the flexibility of URK and define a distribution shifter $h$ that transforms $\omega$. We can start from a standard normal distribution, which is very easy to sample from, and send the samples through $h$ to simulate a much more complex distribution with minimal computation resources required. If we choose $h$ carefully so that the identity function is in its hypothesis set, we will presumably learn a kernel at least as good as the Gaussian kernel.

If we take a step further beyond DNNs, the function $g$ in URK can be assigned some replication policy that creates randomized versions of ${\bf x}$ given different $\omega$ such as multiplying or adding random Gaussian noise to the input. Combined with the idea of being distance-aware in some latent space, we present another architecture as the leftmost diagram in Figure \ref{fig:ker}. We train a GP with URK of these two architectures on a step function and show their predictions in Figure \ref{fig:ker}, where $f$ is a very small DNN. The green curve shows the predictions. The light red area is a 95\% confidence interval based on variance. The blue points are the training data. The replicate policy (the upper right one), although introduces no additional parameters, further increases the model capacity, and its prediction is more reasonable than just using a DNN and a stationary kernel.

\begin{figure}
  \centering
  \includegraphics[width=0.5\textwidth]{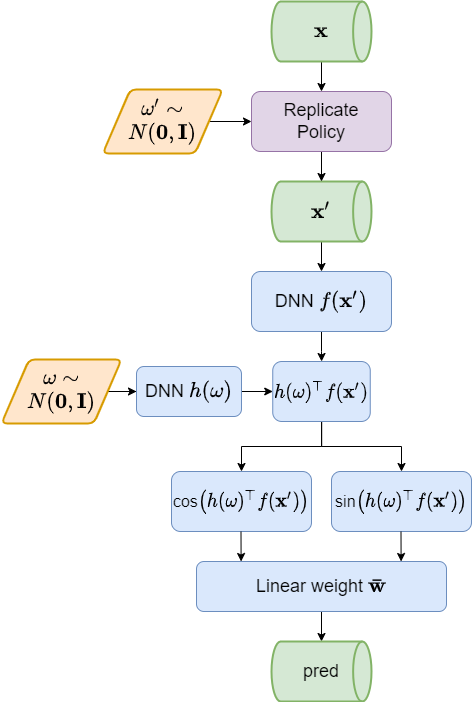}\includegraphics[width=0.5\textwidth]{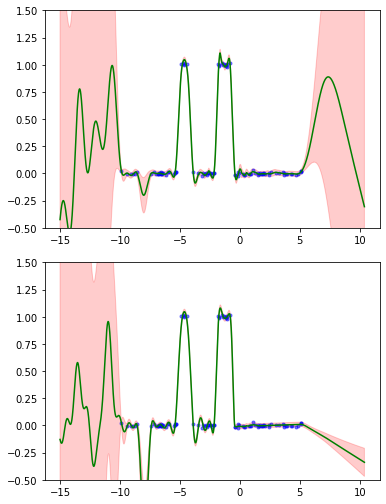}
  \caption{Left: another URK architecture. Right: Train a GP with URK on a step function. Top right: results of the rightmost architecture in Figure \ref{fig:app}. Bottom right: results of the leftmost architecture in Figure \ref{fig:ker}. Blue points: training data; green curve: point estimation; red range: 95\% confidence interval.}
  \label{fig:ker}
\end{figure}

\section{Experiment details}

\begin{figure}
  \centering
  \includegraphics[width=0.9\textwidth]{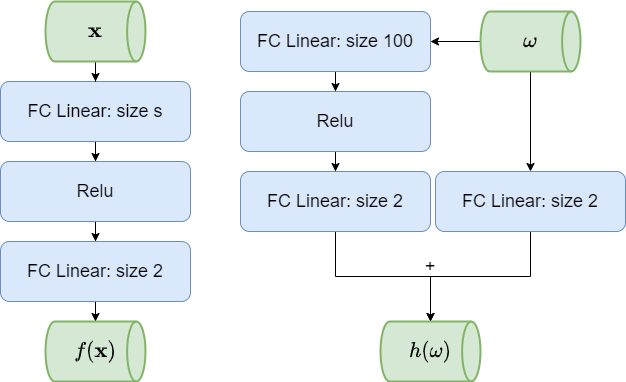}
  \caption{Architecture of DNNs used in UCI experiments. Left: the feature extractor for $\bf{x}$; right: the distribution shifter for $\omega$. FC is short for "fully connected". The layer size $s$ varies for different datasets.}
  \label{fig:exp}
\end{figure}

Figure \ref{fig:exp} shows the architecture of DNNs used in the UCI experiment. For FedAvg, FedProx, and pFedGP, we used the left DNN as their feature extractor. For our methods, we used the rightmost architecture in Figure \ref{fig:app} with the same feature extractor $f$, a very small distribution shifter $h$ (the right DNN in Figure \ref{fig:exp}), and $\omega \sim N({\bf 0}, {\bf I}_{5})$.

We run each random seed of each dataset on 1 CPU and 1 NVIDIA T4 GPU with 16GB RAM.  Some important hyperparameters are listed in Table \ref{Tab:Hyp1}. These hyperparameters are selected through grid searching, as suggested by FedProx.

\renewcommand{\thetable}{A.\arabic{table}}
\setcounter{table}{0}

\begin{table}
  \centering
  \caption{Hyperparameters used in the UCI experiment. For FedLoc $L = 10\rho$ to run the Proximal ADMM algorithm.}
  \begin{tabular}{lllllllllll}
    \toprule
     & \multicolumn{2}{c}{Skillcraft} & \multicolumn{2}{c}{SML} & \multicolumn{2}{c}{Parkinsons} & \multicolumn{2}{c}{Bike} & \multicolumn{2}{c}{CCPP}\\
    \cmidrule(r){2-3}
    \cmidrule(r){4-5}
    \cmidrule(r){6-7}
    \cmidrule(r){8-9}
    \cmidrule(r){10-11}
    \#clients     & 10 & 100 & 10 & 100 & 10 & 100 & 10 & 100 & 10 & 100 \\
    \midrule
    size $s$ & \multicolumn{2}{c}{200} & \multicolumn{2}{c}{2000} & \multicolumn{2}{c}{2000} & \multicolumn{2}{c}{5000} & \multicolumn{2}{c}{5000}\\
    \midrule
    FedProx $\mu$ & 0.5 & 1.0 & 0.1 & 1.0 & 1.0 & 1.0 & 0.1 & 0.01 & 0.001 & 1.0\\
    FedLoc $\rho$ & 5e3 & 5e4 & 5e3 & 5e4 & 1e3 & 5e3 & 5e4 & 5e4 & 1e5 & 1e5\\
    FedBNR-KD $\alpha$ & 10 & 2 & 1 & 0.5 & 5 & 2 & 5 & 0.5 & 5 & 5\\
    \bottomrule
  \end{tabular}
\label{Tab:Hyp1}
\end{table}

\section{Other metrics}

We include the maximum calibration error (MCE) and the Brier score (BRI) of the UCI experiments in Table \ref{Tab:ce_a1} and \ref{Tab:ce_a2}. MCE measures the (estimated) worst difference between $p\%$ confidence intervals and $p'\%$ test points falling into these intervals. BRI measures the mean squared difference between the confidence and observations, where a test point falling in the CI counts as 1, otherwise 0. We also include the standard error of the mean (SEM) in Table \ref{Tab:SEM}. Our methods still perform better in most of the cases.

\begin{table}
{\small
  \caption{Maximum calibration error (MCE) reported. $\Uparrow$ and $\uparrow$ denote significantly worse results with $p < 0.01$ and $p < 0.05$ respectively; $\Downarrow$ and $\downarrow$ denote significantly better results similarly.}
  \centering
  \begin{tabular}{lllllllllll}
    \toprule
    & \multicolumn{2}{c}{Skillcraft} & \multicolumn{2}{c}{SML} & \multicolumn{2}{c}{Parkinsons} & \multicolumn{2}{c}{Bike} & \multicolumn{2}{c}{CCPP}\\
    \cmidrule(r){2-3}
    \cmidrule(r){4-5}
    \cmidrule(r){6-7}
    \cmidrule(r){8-9}
    \cmidrule(r){10-11}
    \#clients     & 10 & 100 & 10 & 100 & 10 & 100 & 10 & 100 & 10 & 100 \\
    \midrule
    Central GP & \multicolumn{2}{c}{0.03} & \multicolumn{2}{c}{0.16} & \multicolumn{2}{c}{0.42} & \multicolumn{2}{c}{0.16} & \multicolumn{2}{c}{0.40}  \\
    \midrule
    
    pFedGP & 0.77$^\Uparrow$ & 0.65$^\Uparrow$ & 0.82$^\Uparrow$ & 0.83$^\Uparrow$ & 0.93$^\Uparrow$ & 0.91$^\Uparrow$ & 0.78$^\Uparrow$ & 0.73$^\Uparrow$ & 0.94$^\Uparrow$ & 0.93$^\Uparrow$\\
    FedLoc & 0.51$^\Uparrow$ & 0.81$^\Uparrow$ & \bf{0.20}$^\Downarrow$ & 0.44$^\Uparrow$ & 0.54$^\Uparrow$ & 0.76$^\Uparrow$ & 0.42$^\Uparrow$ & 0.28$^\Uparrow$ & 0.39$^\Downarrow$ & 0.95\\
    \midrule
    \textbf{ours}\\
    FedBNR & \bf{0.09} & 0.32$^\Uparrow$ & 0.67$^\Uparrow$ & 0.62$^\Uparrow$ & 0.62$^\Uparrow$ & 0.72$^\Uparrow$ & \bf{0.15} & 0.64$^\Uparrow$ & \bf{0.39}$^\Downarrow$ & \bf{0.31}$^\Downarrow$\\
    FedBNR-KD & 0.10 & \bf{0.11} & 0.33 & \bf{0.35} & \bf{0.47} & \bf{0.48} & \bf{0.15} & \bf{0.17} & 0.49 & 0.50\\
    
    \bottomrule
  \end{tabular}
\label{Tab:ce_a1}
}
\end{table}

\begin{table}
{\small
  \caption{Brier score (BRI) reported. $\Uparrow$ and $\uparrow$ denote significantly worse results with $p < 0.01$ and $p < 0.05$ respectively; $\Downarrow$ and $\downarrow$ denote significantly better results similarly.}
  \centering
  \begin{tabular}{lllllllllll}
    \toprule
    & \multicolumn{2}{c}{Skillcraft} & \multicolumn{2}{c}{SML} & \multicolumn{2}{c}{Parkinsons} & \multicolumn{2}{c}{Bike} & \multicolumn{2}{c}{CCPP}\\
    \cmidrule(r){2-3}
    \cmidrule(r){4-5}
    \cmidrule(r){6-7}
    \cmidrule(r){8-9}
    \cmidrule(r){10-11}
    \#clients     & 10 & 100 & 10 & 100 & 10 & 100 & 10 & 100 & 10 & 100 \\
    \midrule
    Central GP & \multicolumn{2}{c}{0.16} & \multicolumn{2}{c}{0.20} & \multicolumn{2}{c}{0.24} & \multicolumn{2}{c}{0.18} & \multicolumn{2}{c}{0.24}  \\
    \midrule
    
    pFedGP & 0.30$^\Uparrow$ & 0.28$^\Uparrow$ & 0.31$^\Uparrow$ & 0.31$^\Uparrow$ & 0.33$^\Uparrow$ & 0.33$^\Uparrow$ & 0.30$^\Uparrow$ & 0.30$^\Uparrow$ & 0.33$^\Uparrow$  & 0.33$^\Uparrow$ \\
    FedLoc & 0.21 & 0.31$^\Uparrow$ & \bf{0.18}$^\Downarrow$ & 0.25$^\Uparrow$ & 0.26$^\Uparrow$ & 0.30$^\Uparrow$ & \bf{0.19}$^\downarrow$ & \bf{0.16}$^\Downarrow$ & 0.24$^\Downarrow$ & 0.33$^\Uparrow$ \\
    \midrule
    \textbf{ours}\\
    FedBNR & \bf{0.18} & 0.22$^\Uparrow$ & 0.29$^\Uparrow$ & 0.28$^\Uparrow$ & 0.28$^\Uparrow$ & 0.30$^\Uparrow$ & 0.20$^\Uparrow$ & 0.29$^\Uparrow$ & \bf{0.24}$^\Downarrow$ & \bf{0.22}$^\Downarrow$\\
    FedBNR-KD & \bf{0.18} & \bf{0.18} & 0.23 & \bf{0.23} & \bf{0.25} & \bf{0.26} & 0.19 & 0.20 & 0.25 & 0.26\\
    \bottomrule
  \end{tabular}
\label{Tab:ce_a2}
}
\end{table}

\begin{table}
{\small
  \caption{UCI regression datasets, RMSE $\pm$ standard error of the mean (SEM) reported.}
  \centering
  \Rotatebox{90}{%
  \begin{tabular}{lllllllllll}
    \toprule
    & \multicolumn{2}{c}{Skillcraft} & \multicolumn{2}{c}{SML} & \multicolumn{2}{c}{Parkinsons} & \multicolumn{2}{c}{Bike} & \multicolumn{2}{c}{CCPP}\\
    \cmidrule(r){2-3}
    \cmidrule(r){4-5}
    \cmidrule(r){6-7}
    \cmidrule(r){8-9}
    \cmidrule(r){10-11}
    \#clients     & 10 & 100 & 10 & 100 & 10 & 100 & 10 & 100 & 10 & 100 \\
    \midrule
    
    local+local & 1.26$\pm$0.03 & 1.48$\pm$0.01 & 1.49$\pm$0.06 & 2.32$\pm$0.03 & 10.9$\pm$0.26 & 10.6$\pm$0.05 & 0.79$\pm$0.02 & 0.91$\pm$0.01 & 14.6$\pm$0.31 & 19.3$\pm$2.08\\
    local+global & 1.08$\pm$0.02 & 1.22$\pm$0.01 & 1.00$\pm$0.02 & 1.65$\pm$0.01 & 6.42$\pm$0.10 & 7.42$\pm$0.01 & 0.59$\pm$0.01 & 0.73$\pm$0.01 & 5.62$\pm$0.19 & 7.03$\pm$0.16\\
    avg+local & 1.05$\pm$0.02 & 1.06$\pm$0.02 & 0.61$\pm$0.06 & 0.81$\pm$0.05 & 9.84$\pm$0.31 & 8.80$\pm$0.11 & 0.45$\pm$0.01 & 0.54$\pm$0.01 & 8.32$\pm$0.43 & 13.4$\pm$0.43\\
    kd+local & 1.04$\pm$0.01 & 1.28$\pm$0.02 & 0.86$\pm$0.06 & 1.34$\pm$0.12 & 6.15$\pm$0.30 & 6.97$\pm$0.27 & 0.51$\pm$0.01 & 0.61$\pm$0.01 & 10.0$\pm$0.61 & 17.1$\pm$0.63\\
    \midrule
    FedAvg & 1.00$\pm$0.01 & 1.03$\pm$0.01 & 0.36$\pm$0.02 & 0.71$\pm$0.02 & 6.83$\pm$0.16 & 7.38$\pm$0.21 & \bf{0.38}$\pm$0.01 & 0.42$\pm$0.01 & 4.45$\pm$0.06 & 4.44$\pm$0.04\\
    FedProx & 0.98$\pm$0.01 & 1.05$\pm$0.01 & 0.34$\pm$0.01 & 0.62$\pm$0.02 & 6.32$\pm$0.27 & 7.38$\pm$0.25 & 0.39$\pm$0.01 & 0.42$\pm$0.01 & 4.43$\pm$0.04 & 4.47$\pm$0.02\\
    pFedGP & 0.99$\pm$0.01 & 1.15$\pm$0.01 & 0.75$\pm$0.05 & 1.34$\pm$0.04 & 9.36$\pm$0.16 & 8.91$\pm$0.09 & 0.45$\pm$0.02 & 0.46$\pm$0.01 & 14.2$\pm$0.61 & 18.2$\pm$0.23\\
    FedLoc & 1.08$\pm$0.01 & 4.15$\pm$0.26 & 2.83$\pm$0.07 & 5.43$\pm$0.04 & 8.40$\pm$0.02 & 11.5$\pm$0.03 & 0.64$\pm$0.01 & 0.75$\pm$0.01 & 20.6$\pm$0.74 & 44.1$\pm$0.61\\
    \midrule
    \textbf{ours}\\
    FedBNR & 0.98$\pm$0.01 & \bf{0.97}$\pm$0.01 & \bf{0.25}$\pm$0.01 & \bf{0.44}$\pm$0.02 & \bf{3.10}$\pm$0.23 & 5.42$\pm$0.20 & 0.39$\pm$0.01 & \bf{0.42}$\pm$0.01 & 4.40$\pm$0.01 & 4.51$\pm$0.03\\
    FedBNR-KD & \bf{0.96}$\pm$0.01 & 0.98$\pm$0.01 & 0.55$\pm$0.01 & 0.55$\pm$0.01 & 4.58$\pm$0.05 & \bf{4.67}$\pm$0.04 & 0.43$\pm$0.01 & 0.48$\pm$0.01 & \bf{4.38}$\pm$0.01 & \bf{4.38}$\pm$0.01\\
    
    \bottomrule
  \end{tabular}
  }
\label{Tab:SEM}
}
\end{table}


\end{document}